\documentclass[lettersize,journal]{IEEEtran}
\usepackage{amsmath,amssymb,amsfonts,amsthm,bm}
\usepackage{algorithm}
\usepackage{array}
\usepackage[caption=false,font=normalsize,labelfont=sf,textfont=sf]{subfig}
\usepackage{textcomp}
\usepackage{stfloats}
\usepackage{url}
\usepackage{verbatim}
\usepackage{graphicx}
\usepackage{cite}
\usepackage[T1]{fontenc}
\usepackage{booktabs,makecell,multirow}
\usepackage{algpseudocode}
\usepackage{hyperref}
\usepackage{xcolor}
\usepackage{braket} % \ket, \bra, \braket
\usepackage{ifthen}

\usepackage[normalem]{ulem} % for \sout in revision marks
\usepackage{multirow}    % 支持\multirow命令
\usepackage{arydshln}  % 支持\cdashline（虚线表格线）
\usepackage{subfig}
\usepackage{orcidlink} %调包

\hypersetup{
    colorlinks=true,
    linkcolor=blue,
    urlcolor=blue,
    citecolor=blue
}

\usepackage{booktabs}
\usepackage{url}
\usepackage{authblk}

\theoremstyle{plain}

\begin{document}

\title{Spectral Complex Autoencoder Pruning:\\
A Fidelity-Guided Criterion for Extreme Structured Channel Compression}
\author{Wei Liu, Xing Deng, Haijian Shao, Fei Wang and Yingtao Jiang
\thanks{Wei Liu, Xing Deng, Haijian Shao, Fei Wang are the College of Computer, Jiangsu University of Science and Technology, Zhenjiang, 212003, China. (e-mail: wliu@stu.just.edu.cn; xdeng@just.edu.cn; jsj\_shj@just.edu.cn; feiw@just.edu.cn)}
\thanks{Yingtao Jiang is with Department of Electrical and Computer Engineering, University of Nevada, Las Vegas, 89115, USA. (e-mail: yingtao.jiang@unlv.edu)}
\thanks{Corresponding author: Xing Deng (e-mail: xdeng@just.edu.cn).}}

% The paper headers (running head)
%%\markboth{IEEE Transactions on Pattern Analysis and Machine Intelligence}%
%%{Shao \MakeLowercase{\textit{et al.}}: Spectral Complex Autoencoder Pruning:\\
%%A Fidelity-Guided Criterion for Extreme Structured Channel Compression}

\maketitle

\begin{abstract}
We propose \emph{Spectral Complex Autoencoder Pruning} (SCAP), a structured channel pruning framework that estimates channel redundancy from input--output interaction descriptors. For each convolutional layer, SCAP constructs a \emph{complex interaction field} by pairing the full multi-channel input activation as the real component with a single spatially aligned output-channel activation as the imaginary component. The interaction field is transformed into the frequency domain, where a low-capacity autoencoder is trained to reconstruct normalized spectra. The resulting reconstruction fidelity provides a scale-invariant cue of descriptor compressibility: easily reconstructed interaction fields tend to be more regular under the learned spectral model, whereas hard-to-reconstruct fields indicate less compressible channel-specific interactions. Since this fidelity cue does not encode the absolute magnitude of the corresponding filter, SCAP combines the fidelity-derived mismatch with a layer-normalized filter $\ell_1$ term to form the default channel importance score. The fused score is then used for threshold-based structured pruning followed by standard fine-tuning. On VGG16 trained on CIFAR-10, SCAP at threshold $\tau=0.6$ achieves $91.27\%$ FLOP reduction and $96.32\%$ parameter reduction with a $1.98\%$ Top-1 accuracy drop from a $93.44\%$ baseline. Empirical results on CIFAR-10, CIFAR-100, and ImageNet-1k, together with an additional DeiT-Tiny/CIFAR-100 validation, show that SCAP provides a competitive accuracy--compression trade-off under aggressive structured pruning and can be instantiated beyond conventional CNN backbones.
\end{abstract}


\section{Introduction}
\IEEEPARstart{M}{odern} convolutional neural networks (CNNs) remain a strong backbone for visual recognition on resource-constrained devices, yet their deployment is often bottlenecked by arithmetic cost, commonly measured by floating-point operations (FLOPs), and by memory footprint, commonly measured by the number of parameters. Structured \emph{channel pruning}---removing entire convolutional output channels and rewiring subsequent layers accordingly---is particularly attractive because it preserves dense tensor structures and can therefore yield practical speedups on standard hardware without requiring specialized sparse kernels~\cite{l1,he2023structuredsurvey}. However, achieving \emph{extreme} compression, e.g., simultaneously reducing FLOPs and parameters by more than $90\%$, while keeping the post-fine-tuning accuracy drop within a small margin, remains challenging. The core difficulty is not only the mechanics of removing channels, but also the design of a criterion that can reliably identify \emph{functionally redundant} channels under the structural and distributional changes induced by pruning.

A large body of structured pruning work estimates channel importance using computationally efficient proxy scores. Representative examples include magnitude-based criteria that rank filters by weight norms or geometric redundancy~\cite{l1,FPGM}, activation-statistics criteria that exploit output sparsity, feature-map statistics, rank, or propagated neuron influence~\cite{network_trimming,HRank,NISP}, gradient/Taylor criteria that approximate the loss variation caused by removing channels~\cite{molchanov2017pruning}, and sparsity-regularized or scaling-factor methods that learn channel-level gates during training~\cite{wen2016structured,liu2017learning,SSS}. These methods are effective and often attractive in moderate-compression regimes because they are simple, inexpensive, and easy to integrate into existing training pipelines. Nevertheless, under aggressive compression, small inaccuracies in a local proxy score may be amplified into large architectural changes. Recent analyses of pruning criteria have reported two relevant limitations: different criteria can produce highly similar channel rankings, and in some settings, their importance scores may be too close to reliably distinguish redundant from indispensable filters~\cite{huang2021rethinking}. Therefore, rather than assuming that a channel is redundant solely from its weight magnitude, activation level, or instantaneous gradient response, we seek a criterion that explicitly evaluates whether the channel-specific input--output interaction can be compactly reconstructed.

This paper proposes \emph{Spectral Complex Autoencoder Pruning} (SCAP), a reconstruction-based criterion designed to measure redundancy at the level of a channel's \emph{input--output interaction}. Consider a convolutional layer with input activation $X\in\mathbb{R}^{B\times C_{\mathrm{in}}\times H\times W}$ and output activation $Y\in\mathbb{R}^{B\times C_{\mathrm{out}}\times H_1\times W_1}$. For each output channel $k$, SCAP constructs a \emph{complex interaction field}
\[
Z_k \;=\; X \;+\; i\,\mathrm{Broadcast}(\mathrm{Resize}(Y_k)) \;\in\; \mathbb{C}^{B\times C_{\mathrm{in}}\times H\times W},
\]
where $Y_k\in\mathbb{R}^{B\times 1\times H_1\times W_1}$ denotes the $k$-th output feature map. The operator $\mathrm{Resize}(\cdot)$ spatially aligns $Y_k$ to the input resolution $(H,W)$. The operator $\mathrm{Broadcast}(\cdot)$ is a deterministic shape-matching operation: for $A\in\mathbb{R}^{B\times 1\times H\times W}$,
\[
[\mathrm{Broadcast}(A)]_{b,c,h,w}=A_{b,1,h,w},\quad c=1,\ldots,C_{\mathrm{in}},
\]
so the same resized output map is replicated along the input-channel dimension. Thus, the real part preserves the full multi-channel input activation, while the imaginary part injects the response pattern of one specific output channel into the same tensor shape. No additional learnable parameter is introduced by this broadcasting operation.

We then transform $Z_k$ into the frequency domain and train a deliberately low-capacity autoencoder, one per layer, to reconstruct the normalized spectra. At scoring time, the reconstruction fidelity is computed as the absolute cosine similarity between the original and reconstructed interaction fields, with real and imaginary parts concatenated. This fidelity should be understood as a \emph{compressibility cue}, not as a complete importance measure. A high-fidelity channel has an interaction field that is easier to reproduce by the constrained spectral autoencoder, suggesting that its descriptor is more regular or more compatible with the low-dimensional structure learned for that layer. A low-fidelity channel, in contrast, contains interaction patterns that are harder for the autoencoder to reconstruct and may therefore carry less redundant information. This interpretation is relative to the learned autoencoder and the calibration data; it does not by itself guarantee that removing a high-fidelity channel will preserve accuracy. Because SCAP scores structural units layer by layer, it should be regarded as a local channel-ranking criterion rather than a global structural optimizer or a guarantee of end-to-end accuracy preservation.

This intuition can be stated more formally from a manifold-learning viewpoint. Let $\mathcal{M}_\ell$ denote the set of typical spectral interaction fields for layer $\ell$ induced by the calibration data and the current network. A sufficiently constrained autoencoder trained on samples from $\mathcal{M}_\ell$ tends to reconstruct directions that are common and compatible with its bottleneck, while producing larger errors on directions that are rare, irregular, or require higher intrinsic dimensionality than the bottleneck permits. Therefore, the reconstruction fidelity provides a model-dependent estimate of descriptor compressibility: high fidelity indicates that the corresponding interaction field is well represented by the low-dimensional structure captured by the autoencoder, whereas low fidelity indicates that the field is harder to explain under this constrained representation. In SCAP, this estimate is used as evidence of channel redundancy relative to the learned spectral interaction model, rather than as a standalone guarantee of post-pruning accuracy preservation. Since the fidelity is computed as an absolute cosine similarity between normalized interaction descriptors, it is inherently scale-invariant: it reflects the reconstructability of the shape or structure of a channel-specific interaction field, but it does not measure the absolute amplitude or energy carried by the corresponding filter. Thus, fidelity alone is a compressibility cue rather than a complete saliency measure. A large-amplitude but spectrally simple channel may still obtain high reconstruction fidelity, although pruning it could noticeably perturb downstream activations. For this reason, the full SCAP pruning rule combines the fidelity-derived score with a normalized filter $\ell_1$ term, which provides magnitude-aware stabilization. A small minimum-keep constraint is further used only to avoid pathological layer collapse under fixed-threshold pruning; it is not used to enforce a target compression ratio and is not triggered in our reported experiments.

Beyond the criterion itself, a practical obstacle is memory. Na\"ively scoring all channels by expanding a batch to size $B\cdot C_{\mathrm{out}}$ is prohibitive under limited GPU memory. SCAP is implemented with strict memory discipline: we process output channels sequentially (or in small groups), use gradient accumulation for autoencoder training, and explicitly release intermediate tensors (FFT outputs, reconstructions) to keep peak memory bounded. This allows training and scoring on commodity GPUs (e.g., 11\,GB) without sacrificing the conceptual model. Our principal contributions are as follows:
\begin{itemize}
\item We introduce a complex input--output interaction descriptor that couples multi-channel inputs with per-channel outputs, enabling redundancy analysis beyond isolated weight, activation, or gradient proxies.
\item We propose a complete pruning criterion that combines spectral interaction reconstructability with magnitude-aware saliency. The fidelity-derived mismatch captures descriptor-level compressibility, while the normalized set-$\ell_1$ term compensates for the scale-invariance of fidelity.
\item We provide a memory-efficient implementation based on per-channel processing, gradient accumulation, and explicit tensor release, and empirically validate SCAP on CIFAR benchmarks, ImageNet-1k, and an additional DeiT-Tiny non-CNN setting.
\end{itemize}

\section{Related Work}

\subsection{Structured and channel pruning for CNN acceleration}
Early network pruning studies removed individual weights according to saliency estimates derived from second-order approximations~\cite{lecun1990obd,hassibi1993obs}. Later compression pipelines combined pruning with quantization and entropy coding to reduce storage and computation~\cite{han2015deep,han2015weights}. However, unstructured sparsity usually requires specialized sparse kernels or hardware support to obtain actual acceleration. In contrast, structured pruning removes complete filters, channels, or blocks, producing dense subnetworks that are easier to accelerate on standard hardware~\cite{l1,he2023structuredsurvey}. This deployment advantage makes channel pruning an important direction for practical CNN compression.

Existing structured pruning methods can be broadly grouped according to the information used to estimate channel importance. Magnitude-based methods rank filters by weight norms or geometric redundancy. The classical $\ell_1$-norm criterion removes filters with small weight magnitude~\cite{l1}, while FPGM measures redundancy through the distance between filters and a geometric median~\cite{FPGM}. These criteria are simple and inexpensive, but they mainly inspect the filter tensor itself and may not fully reflect how the filter interacts with the input activation distribution. Activation-statistics methods instead use data-dependent responses, such as the average percentage of zero activations~\cite{network_trimming}, propagated neuron influence~\cite{NISP}, or the rank of generated feature maps~\cite{HRank}. These methods better reflect the behavior of trained networks on data, yet their statistics are still usually summarized into local channel scores. Gradient- and Taylor-based methods approximate the loss change caused by removing a channel~\cite{molchanov2017pruning}, but their estimates can depend on the current mini-batch, optimization state, and the quality of local first-order approximations. Another family learns sparse scaling factors or gates through regularization~\cite{wen2016structured,liu2017learning,SSS}, which can produce compact structures but usually requires additional training or carefully balanced sparsity penalties.

These proxy-based criteria have been successful, especially at moderate pruning ratios. Nevertheless, high-compression regimes make channel selection more fragile: when most channels must be removed, small score differences can decide large architectural changes. Recent analysis has shown that several widely used pruning criteria may suffer from ``similarity'' and ``applicability'' blind spots: different criteria may yield nearly identical rankings, and some score distributions may be too close to reliably separate redundant from indispensable filters~\cite{huang2021rethinking}. Recent pruning studies further revisit channel selection from information-theoretic, global-allocation, multidimensional, and dependency-aware perspectives. CICC estimates pruning behavior through the concentration of information and channel contributions~\cite{cicc2022}, while information-theoretic criteria evaluate filters or layers using entropy/rank-based information concentration, information capacity, independence, or related contribution measures~\cite{Chen,fpici2025}. Other recent studies emphasize global channel pruning for resource-limited platforms~\cite{three_stage_tnnls2024}, multidimensional pruning across channel, spatial, and temporal redundancies~\cite{mdp_tnnls2024}, and self-adaptive Taylor pruning under high pruning ratios~\cite{saaf_tnnls2025}. Shapley-based pruning further considers channel dependency grouping and data synthesis for estimating cooperative channel contributions~\cite{shapley2025}. These methods improve pruning-budget allocation, information modeling, or dependency estimation. In contrast, SCAP focuses on a different signal: the spectral reconstructability of a complex input--output interaction descriptor, stabilized by normalized filter magnitude.

\subsection{Reconstruction- and distillation-based pruning}
Reconstruction-based pruning aims to preserve the network function more explicitly. Representative methods select channels by minimizing the reconstruction error of intermediate or next-layer feature maps~\cite{ThiNet,CP}. Related compression methods approximate convolutional layers with low-rank or factorized structures under reconstruction objectives~\cite{denton2014exploiting,jaderberg2014speeding,lebedev2015speeding}. Knowledge distillation provides another function-preserving signal by transferring logits or intermediate representations from a teacher network to a compact student~\cite{hinton2015distilling,romero2015fitnets,zagoruyko2017attention}. These approaches are appealing because they directly consider functional preservation rather than only static parameter statistics. SCAP follows the same broad philosophy that pruning should preserve function, but differs in both the object being reconstructed and the interpretation of reconstruction quality. 

\subsection{Spectral-domain analysis and complex-valued representations}
Frequency-domain representations have been widely used to analyze and process visual signals. Scattering networks and spectral pooling show that frequency-domain operations can expose stable or compact structures in visual representations~\cite{bruna2013invariant,rippel2015spectralpool}, while FFT-based convolution exploits the convolution theorem for acceleration~\cite{mathieu2013fast}. More recent studies on Fourier features and periodic activations further demonstrate that spectral representations can help neural networks model high-frequency information~\cite{tancik2020fourier,sitzmann2020siren}. In parallel, complex-valued neural networks provide mechanisms for representing magnitude- and phase-sensitive interactions~\cite{trabelsi2018deepcomplex,arjovsky2016unitary}.

Unlike FFT-based convolution acceleration or complex-valued network substitution, SCAP leaves the original CNN architecture intact. The complex interaction field serves purely as an analysis tool: the input activation forms the real component, a selected output-channel response forms the imaginary component, and the resulting field is transformed into the frequency domain to evaluate spectral reconstructability. This design enables the pruning criterion to jointly capture spatial response patterns and their spectral compactness.

\subsection{Structural interaction modeling and multi-scale feature selection}
Recent studies on structural estimation and multi-scale feature selection suggest that reliable selection should exploit relational and scale-aware cues rather than isolated scalar descriptors. For example, neighborhood-guided inlier estimation uses local geometric consistency to identify reliable correspondences~\cite{yuan2025onn}, while multi-task multi-scale feature selection improves robustness by transferring information across related low-dimensional tasks~\cite{wu2025multitask}. Although these studies do not directly address neural network pruning, they support the broader motivation of SCAP: reliable selection can benefit from structured interactions, geometric consistency, and scale-aware feature behavior. SCAP instantiates this idea by constructing a spectral input--output interaction field and evaluating its reconstructability for pruning.

% =====================================================================
\section{Method}
\label{sec:method}
We first give an overview of the proposed pipeline in figure~\ref{fig:overview}. SCAP consists of two stages executed layer-by-layer over convolutional layers: 
(i) a training stage that learns a layer-specific low-capacity spectral reconstructor with component-specific tiny MLP branches for the real and imaginary spectral components of a channel-wise interaction descriptor, and (ii) a pruning stage that scores each output channel by the reconstructability of its descriptor and performs threshold-based structured pruning. The fidelity-derived score is used as a descriptor-level compressibility cue and is combined with a layer-normalized set-$\ell_1$ magnitude term in the default pruning rule. The rest of this section formalizes the setup and details each module in figure~\ref{fig:overview}.

\begin{figure*}[t]
\centering
\includegraphics[width=\textwidth]{figures/liu-fig1.pdf}
\caption{Overview of SCAP. For each convolutional layer and output channel, SCAP constructs a complex interaction descriptor from the input activation and the resized-and-broadcast output response. The descriptor is first transformed into the frequency domain and then normalized in the spectral domain, with the real and imaginary components processed separately. A layer-specific low-capacity spectral reconstructor, implemented as two component-specific tiny MLP branches shared across input channels, is trained to reconstruct the normalized real and imaginary spectral components. During pruning, SCAP computes the absolute-cosine fidelity and derives the fidelity-based importance term $I_{\mathrm{fid}}=1-\mathrm{Fid}$ as a scale-invariant reconstructability cue. This term is fused with a layer-normalized set-$\ell_1$ magnitude term to obtain the final channel-importance score, which is then used for threshold-based structured pruning followed by fine-tuning.}
\label{fig:overview}
\end{figure*}

\subsection{Problem setup}
\label{sec:setup}
Consider a pretrained CNN $f_{\theta}$ with convolutional layers. For a convolutional layer $\ell$, let the input activation be
$X^{(\ell)} \in \mathbb{R}^{B \times C_{\mathrm{in}} \times H \times W}$
and the output activation be
$Y^{(\ell)} \in \mathbb{R}^{B \times C_{\mathrm{out}} \times H_1 \times W_1}$.
Our goal is \emph{structured} pruning: we remove a subset of output channels (filters) and rewire subsequent layers accordingly, producing a smaller network $\tilde f_{\tilde\theta}$ with reduced FLOPs and parameters while preserving accuracy after fine-tuning.

A central design choice is how to quantify channel redundancy. In this work, we estimate redundancy from the reconstructability of a channel-specific input--output interaction descriptor in the frequency domain. This reconstructability is not used as a standalone guarantee of channel dispensability; rather, it provides a layer-local compressibility cue. We instantiate this idea by constructing a channel-wise complex interaction descriptor for each output channel, training a low-capacity spectral reconstructor, and combining the resulting fidelity-derived score with a normalized set-$\ell_1$ magnitude term before threshold-based pruning. This layer-local formulation is used to obtain tractable channel rankings, while the final effect of pruning remains network-level and is handled empirically through post-pruning fine-tuning.

\subsection{Complex interaction field}
\label{sec:cif}
For each output channel $k \in \{1,\dots,C_{\mathrm{out}}\}$ of layer $\ell$, we define a \emph{complex interaction field}
\begin{equation}
\begin{split}
Z^{(\ell)}_{k} \;=\; X^{(\ell)} \;+\; i \cdot \mathrm{Broadcast}\!\Big(\mathrm{Resize}\big(Y^{(\ell)}_{k}\big)\Big)\\
\;\in\; \mathbb{C}^{B \times C_{\mathrm{in}} \times H \times W},
\label{eq:complex_field}
\end{split}
\end{equation}
where $Y^{(\ell)}_{k} \in \mathbb{R}^{B \times 1 \times H_1 \times W_1}$ is the $k$-th output feature map. The operator $\mathrm{Resize}(\cdot)$ aligns $Y^{(\ell)}_{k}$ to spatial size $(H,W)$ using bilinear interpolation in our implementation. Given the resized map $A\in\mathbb{R}^{B\times 1\times H\times W}$, $\mathrm{Broadcast}(A)$ replicates $A$ along the channel dimension to match the input activation:
\begin{equation}
[\mathrm{Broadcast}(A)]_{b,c,h,w}=A_{b,1,h,w},\quad c=1,\ldots,C_{\mathrm{in}}.
\label{eq:broadcast}
\end{equation}
Therefore, $Z^{(\ell)}_{k}$ pairs the full input activation with the spatial response of a specific output channel in a common complex-valued tensor. The broadcasting step is used only for tensor alignment and does not add parameters or change the output response itself.

\paragraph{Why complex-valued coupling?}
Eq.~\eqref{eq:complex_field} couples the layer input and a specific output channel into one structured object. Since redundancy is defined relative to the input--output transformation, $Z_k$ contains relational information that is not available from weights alone, activations alone, or a single scalar saliency score. This design allows SCAP to evaluate whether the interaction pattern induced by a channel is compactly reconstructable in the spectral domain.

\subsection{Spectral transform and normalization}
\label{sec:spectral_norm}

For each convolutional layer $\ell$ and output channel $k$, we first construct the complex interaction field $Z^{(\ell)}_k$ as defined in Eq.~(1). We then apply a 2D Fourier transform to this complex field:
\begin{equation}
F^{(\ell)}_k
=
\mathcal{F}\!\left(Z^{(\ell)}_k\right)
\in \mathbb{C}^{B\times C_{\mathrm{in}}\times H\times W},
\label{eq:fft_transform}
\end{equation}
where $\mathcal{F}(\cdot)$ denotes the 2D FFT applied over the spatial dimensions. We denote the real and imaginary spectral components by
\[
F^{(\ell),R}_k=\Re\!\left(F^{(\ell)}_k\right),
\qquad
F^{(\ell),I}_k=\Im\!\left(F^{(\ell)}_k\right).
\]

After the spectral transform, we normalize the two spectral components separately:
\begin{equation}
\widetilde{F}^{(\ell),R}_k
=
\frac{F^{(\ell),R}_k-\mu^R_\ell}{\sigma^R_\ell+\epsilon},
\qquad
\widetilde{F}^{(\ell),I}_k
=
\frac{F^{(\ell),I}_k-\mu^I_\ell}{\sigma^I_\ell+\epsilon},
\label{eq:spectral_norm}
\end{equation}
where $\mu^R_\ell,\sigma^R_\ell$ and $\mu^I_\ell,\sigma^I_\ell$ are the layer-specific mean and standard deviation estimated from calibration descriptors for the real and imaginary spectral components, respectively, and $\epsilon$ is a small constant for numerical stability.

This normalization is performed after the FFT and is used to stabilize the distribution of spectral descriptors before reconstruction. The same spectral preprocessing is applied during both autoencoder training and pruning-time scoring. The normalized spectral descriptors $\widetilde{F}^{(\ell),R}_k$ and $\widetilde{F}^{(\ell),I}_k$ are then used as the inputs to the layer-specific low-capacity reconstructor described in Section~\ref{sec:ae}.

\subsection{Tiny spectral autoencoder}
\label{sec:ae}
For each layer $\ell$, we train a layer-specific low-capacity reconstructor 
$g_{\phi_\ell}$ that maps the normalized spectra 
$(\widetilde{F}^{R},\widetilde{F}^{I})$ to their reconstructed normalized counterparts 
$(\widehat{\widetilde{F}}^{R},\widehat{\widetilde{F}}^{I})$. The capacity is deliberately restricted so that the reconstructor mainly captures recurring spectral structures observed in the calibration descriptors, instead of trivially memorizing all channel-specific patterns. In the implementation used for all reported experiments, $g_{\phi_\ell}$ consists of 
two component-specific tiny MLP branches, denoted by $h^{R}_{\ell}$ and $h^{I}_{\ell}$,  for the real and imaginary spectral components, respectively. These two branches do  not share parameters with each other. However, each branch is shared across input  channels by reshaping the channel dimension into the batch dimension. Specifically, after spectral normalization, we reshape the real component as
\begin{equation*}
\begin{split}
\text{reshape } \widetilde F^{R} \in 
\mathbb{R}^{B\times C_{\mathrm{in}}\times H\times W} \\
\;\mapsto\; 
U^{R}\in\mathbb{R}^{(B\cdot C_{\mathrm{in}})\times N}, 
\;\; N=H\cdot W,
\end{split}
\end{equation*}
and similarly reshape $\widetilde F^{I}$ to $U^{I}$. We then apply the two 
component-specific channel-shared MLP branches:
\[
\widehat U^{R}=h^{R}_{\ell}(U^{R}),\quad 
\widehat U^{I}=h^{I}_{\ell}(U^{I}),
\]
and reshape the outputs back to 
$\widehat{\widetilde F}^{R},\widehat{\widetilde F}^{I}\in
\mathbb{R}^{B\times C_{\mathrm{in}}\times H\times W}$.
This design is \emph{channel-adaptive}: it supports arbitrary 
$C_{\mathrm{in}}$ because channels are batched, and the parameter count of each 
branch depends primarily on $N=H\cdot W$ rather than $C_{\mathrm{in}}$.

A concrete instantiation of each branch is:
\begin{equation}
h^{q}_{\ell}(u)=
\mathrm{Tanh}\!\Big(
W^{q}_{2,\ell}\,\mathrm{ReLU}(W^{q}_{1,\ell}u)
\Big),
\quad q\in\{R,I\},
\label{eq:tiny_mlp}
\end{equation}
where $W^{q}_{1,\ell}\in\mathbb{R}^{\frac{N}{4}\times N}$ and 
$W^{q}_{2,\ell}\in\mathbb{R}^{N\times \frac{N}{4}}$.

\paragraph{Training objective.}
For each layer $\ell$, we train $g_{\phi_\ell}$ using mean squared error in the normalized spectral domain:
\begin{equation}
\mathcal{L}_{\mathrm{AE}}^{(\ell)}
=\mathbb{E}_{k,b}\Big[\big\|\widehat{\widetilde F}^{(\ell),R}_{k}-\widetilde F^{(\ell),R}_{k}\big\|_2^2
+\big\|\widehat{\widetilde F}^{(\ell),I}_{k}-\widetilde F^{(\ell),I}_{k}\big\|_2^2\Big].
\label{eq:ae_loss}
\end{equation}
We implement this efficiently by iterating over output channels $k$ (or small groups) and using gradient accumulation to emulate a larger batch under limited GPU memory.

\subsection{Fidelity score and importance}
\label{sec:fidelity}
After training, we freeze $g_{\phi_\ell}$ and score each output channel by its reconstruction fidelity. For a mini-batch, we apply the same spectral normalization as in Eq.~\eqref{eq:spectral_norm}, feed the standardized spectra $(\widetilde F^{R},\widetilde F^{I})$ into $g_{\phi_\ell}$ to obtain reconstructed standardized spectra $(\widehat{\widetilde F}^{R},\widehat{\widetilde F}^{I})$, and then undo the standardization:
\begin{align}
\widehat F^{R} &= (\sigma_R+\varepsilon)\widehat{\widetilde F}^{R} + \mu_R, \nonumber\\
\widehat F^{I} &= (\sigma_I+\varepsilon)\widehat{\widetilde F}^{I} + \mu_I, \qquad
\widehat F=\widehat F^{R}+i\widehat F^{I}.
\label{eq:denorm}
\end{align}
We then invert the transform to obtain a reconstructed interaction field:
\begin{equation}
\widehat Z^{(\ell)}_{k}=\mathrm{iFFT2}\!\left(\widehat F^{(\ell)}_{k}\right).
\label{eq:ifft}
\end{equation}
Let $\mathrm{vec}(\cdot)$ flatten all tensor entries, and define the real-valued embedding
\begin{equation*}
\begin{split}
v^{(\ell)}_{k,b}=\mathrm{vec}\big([\Re(Z^{(\ell)}_{k,b}),\Im(Z^{(\ell)}_{k,b})]\big)\in\mathbb{R}^{D},\quad\\
\widehat v^{(\ell)}_{k,b}=\mathrm{vec}\big([\Re(\widehat Z^{(\ell)}_{k,b}),\Im(\widehat Z^{(\ell)}_{k,b})]\big).
\end{split}
\end{equation*}

We compute the per-sample absolute cosine similarity and average across the batch:
\begin{equation}
\mathrm{Fid}^{(\ell)}_{k}
=\frac{1}{B}\sum_{b=1}^B
\left|
\frac{\langle v^{(\ell)}_{k,b},\widehat v^{(\ell)}_{k,b}\rangle}
{\|v^{(\ell)}_{k,b}\|_2\;\|\widehat v^{(\ell)}_{k,b}\|_2}
\right|
\in [0,1].
\label{eq:fidelity}
\end{equation}
\paragraph{Normalized fidelity--error identity.}
Define normalized vectors $u=v/\|v\|_2$ and $\hat u=\widehat v/\|\widehat v\|_2$. Consistent with the fidelity definition in Eq.~\eqref{eq:fidelity}, let $F=|\langle u,\hat u\rangle|$ denote the absolute cosine fidelity. Then there exists $s\in\{+1,-1\}$ such that $\langle u,s\hat u\rangle=F$ and
\begin{equation}
\|u-s\hat u\|_2^2 = 2(1-F).
\label{eq:fidelity_error_identity}
\end{equation}
\emph{Proof.} Choose $s=\mathrm{sign}(\langle u,\hat u\rangle)$. Since $\|u\|_2=\|\hat u\|_2=1$,
\[
\|u-s\hat u\|_2^2=\|u\|_2^2+\|\hat u\|_2^2-2\langle u,s\hat u\rangle=2-2F=2(1-F).
\]
\hfill$\square$

Eq.~\eqref{eq:fidelity_error_identity} shows that high fidelity is equivalent to small reconstruction error in normalized $\ell_2$ geometry, up to a global sign. Therefore, $1-\mathrm{Fid}_k$ can be interpreted as a scale-invariant reconstruction-mismatch score for the channel-specific spectral interaction descriptor. This score captures whether the input--output interaction pattern of a channel can be compactly reconstructed by the low-capacity spectral model, but it does not encode the absolute magnitude or energy of the corresponding filter. Conversely, a pure $\ell_1$-magnitude score captures amplitude information but ignores interaction-level reconstructability. Hence, SCAP uses fidelity-derived mismatch and normalized filter $\ell_1$ magnitude as complementary components of the final pruning criterion. Detailed theoretical rationale and proofs are provided in the supplementary material.

\paragraph{Importance score.}
We first define the fidelity-derived importance as
\begin{equation}
I^{(\ell)}_{k,\mathrm{fid}} = 1 - \mathrm{Fid}^{(\ell)}_{k}.
\label{eq:importance_fid}
\end{equation}
A larger $I^{(\ell)}_{k,\mathrm{fid}}$ indicates that the channel-specific interaction field is harder to reconstruct by the low-capacity spectral autoencoder and is therefore considered less compressible.

However, $I^{(\ell)}_{k,\mathrm{fid}}$ is derived from cosine fidelity and is therefore scale-invariant. It captures normalized reconstructability of the interaction descriptor, but it does not encode the absolute magnitude of the corresponding filter. This distinction is important in aggressive pruning: a channel may be spectrally simple and thus highly reconstructable, while still carrying a large-amplitude response that should not be removed solely based on fidelity. To compensate for this limitation, we introduce a normalized filter-magnitude term
\begin{equation}
I^{(\ell)}_{k,\mathrm{l1}}=
\frac{\|W^{(\ell)}_{k}\|_1}{\max_j \|W^{(\ell)}_{j}\|_1+\varepsilon},
\label{eq:l1_importance}
\end{equation}
where $W^{(\ell)}_{k}$ denotes the $k$-th convolutional filter. The final importance score is computed by additive fusion:
\begin{equation}
I^{(\ell)}_{k}
=
\Psi_{\mathrm{add}}\!\left(I^{(\ell)}_{k,\mathrm{fid}}, I^{(\ell)}_{k,\mathrm{l1}}\right)
=
\alpha I^{(\ell)}_{k,\mathrm{fid}}
+
(1-\alpha) I^{(\ell)}_{k,\mathrm{l1}},
\label{eq:importance_fusion}
\end{equation}
where $\alpha\in[0,1]$ controls the balance between spectral reconstructability and magnitude awareness. The two extreme cases have clear interpretations: $\alpha=1$ corresponds to fidelity-only pruning, while $\alpha=0$ reduces the score to normalized $\ell_1$-magnitude pruning. For other fusion rules---such as multiplication and exponentiation---additional diagnostic results are provided in the supplementary material. The sensitivity of SCAP to $\alpha$ is analyzed in Section~\ref{sec:alpha_sensitivity}. After fusion, we map $\{I^{(\ell)}_k\}$ to $[0,1]$ within each layer for thresholding.

\subsection{Threshold-based Structured Pruning}
\label{sec:pruning}

Given a fixed threshold $\tau$, SCAP keeps the channels whose normalized importance scores exceed the threshold:
\[
\mathcal{K}_\ell = \{k: I^{(\ell)}_{k,\mathrm{norm}} \ge \tau\}.
\]
A larger $\tau$ leads to stronger compression. We mainly report $\tau\in\{0.5,0.6\}$, where $\tau=0.5$ is a strong but relatively stable setting and $\tau=0.6$ represents a more aggressive compression regime.

To avoid pathological layer collapse under fixed-threshold pruning, we impose a minimum-keep constraint,
\[
|\mathcal{K}_\ell| \ge K_{\min}(\ell).
\]
If this constraint is violated, we retain the top-$K_{\min}(\ell)$ channels according to $I^{(\ell)}_{k,\mathrm{norm}}$. This safeguard is used only for stability, not for tuning the compression ratio, and it is not triggered in our reported experiments. After channel selection, we prune the corresponding convolutional filters and propagate the selection to subsequent layers, yielding a standard dense pruned network for fine-tuning.

% =====================================================================
\section{Experiments}
\label{sec:experiments}
\subsection{Experimental Settings}
\label{sec:exp_settings}

\paragraph{Datasets and evaluation.}
We evaluate SCAP on CIFAR-10, CIFAR-100, and ImageNet-1k. CIFAR-10/100 contain $60\text{K}$ $32\times32$ images with $50\text{K}$ training and $10\text{K}$ test images, and we report Top-1 accuracy on the official test set~\cite{cifar}. Standard data augmentation is used, including random cropping with padding and random horizontal flipping. To validate scalability beyond low-resolution benchmarks, we further evaluate ResNet-50 on ImageNet-1k~\cite{imagenet}, which contains approximately 1.28M training images and 50K validation images from 1000 classes. For ImageNet-1k, we report single-crop Top-1/Top-5 validation accuracy, FLOP reduction, and parameter reduction.

\paragraph{Backbones and baselines.}
For CIFAR-10/100, we use four standard structured-pruning backbones: VGG16~\cite{vgg}, ResNet-56, ResNet-110~\cite{resnet}, and DenseNet-40~\cite{densenet}. For ImageNet-1k, we use ResNet-50~\cite{resnet}. The baseline accuracy, FLOPs, and parameter counts of all reference models are summarized in Table~\ref{tab:baseline_models}. All pruning results are reported relative to the corresponding baseline checkpoint of the same dataset and architecture, so that the comparison focuses on the accuracy--compression trade-off rather than absolute accuracy differences caused by different starting points. To further examine architectural generality beyond CNNs, the supplementary material reports an additional DeiT-Tiny/CIFAR-100 experiment, where SCAP is instantiated on Transformer FFN expansion channels.

\paragraph{Training and fine-tuning protocol.}
For CIFAR-10/100, all baseline models are trained for 200 epochs using SGD with momentum $0.9$, weight decay $5\times10^{-4}$, batch size $256$, and initial learning rate $0.1$, decayed by a factor of $10$ every 50 epochs. After pruning, the networks are fine-tuned for 100 epochs with the same optimizer and batch size, using an initial learning rate of $0.01$ decayed every 30 epochs. For ImageNet-1k, we follow the standard ResNet-50 pruning evaluation protocol used in prior structured pruning works: the pruned model is fine-tuned on the ImageNet training set and evaluated on the validation set using single-crop Top-1/Top-5 accuracy. Unless otherwise stated, the same pruning threshold and scoring rule are used across CIFAR and ImageNet experiments.

\paragraph{Spectral autoencoder training.}
SCAP trains one spectral autoencoder for each scored convolutional layer. For each layer, we cache calibration activations, construct the complex interaction field in Eq.~\eqref{eq:complex_field}, transform it into the spectral domain using Eq.~\eqref{eq:fft_transform}, apply spectral standardization in Eq.~\eqref{eq:spectral_norm}, and optimize the autoencoder with the reconstruction loss in Eq.~\eqref{eq:ae_loss}. The maximum training budget for each autoencoder is set to 100 epochs, with learning rate $10^{-3}$, weight decay $10^{-5}$, and early stopping using a patience of 5 epochs. 

\paragraph{Implementation and hardware.}
All experiments are implemented in PyTorch. CIFAR experiments are run on a single NVIDIA RTX 2080Ti GPU with 11GB memory, and ImageNet-1k ResNet-50 experiments are run on a single NVIDIA RTX 4090 GPU. SCAP uses per-output-channel processing, gradient accumulation, FP16 activation caching when applicable, and explicit release of intermediate FFT/reconstruction tensors to control peak memory usage.

\begin{table}[t]
    \centering
    \caption{Baseline accuracy, FLOPs, and parameter counts of the reference models. 
    FLOPs are reported as multiply-add operations in millions (M); Params are in millions. 
    For CIFAR-10/100, only Top-1 accuracy is reported. For ImageNet-1k, both Top-1 and Top-5 accuracies are reported.}
    \label{tab:baseline_models}
    \renewcommand{\arraystretch}{1.1}
    \setlength{\tabcolsep}{4pt}
    \begin{tabular}{llccc}
        \toprule
        Dataset & Model & Top-1/5 Acc (\%) & FLOPs (M) & Params (M) \\
        \midrule
        \multirow{4}{*}{CIFAR-10}
        & VGGNet-16   & 93.44 / -- & 314.57 & 14.99 \\
        & ResNet-56   & 93.66 / -- & 130.02 & 0.88  \\
        & ResNet-110  & 92.63 / -- & 259.49 & 1.75  \\
        & DenseNet-40 & 93.72 / -- & 292.50 & 1.06  \\
        \midrule
        \multirow{4}{*}{CIFAR-100}
        & VGGNet-16   & 71.77 / -- & 341.62 & 15.04 \\
        & ResNet-56   & 70.82 / -- & 130.03 & 0.88  \\
        & ResNet-110  & 71.26 / -- & 259.50 & 1.76  \\
        & DenseNet-40 & 72.95 / -- & 292.54 & 1.10  \\
        \midrule
        ImageNet-1k
        & ResNet-50   & 76.15 / 92.87 & 4130.00 & 25.56 \\
        \bottomrule
    \end{tabular}
\end{table}

\subsection{Metrics and reporting}
\label{sec:metrics}
We report FLOP reduction (FR) and parameter reduction (PR) relative to the baseline checkpoint:
\[
\mathrm{PR} = 100\left(1-\frac{P(\tilde f)}{P(f)}\right)\%,\qquad
\mathrm{FR} = 100\left(1-\frac{F(\tilde f)}{F(f)}\right)\%.
\]
For CIFAR-10 and CIFAR-100, accuracy denotes Top-1 test accuracy. 
For ImageNet-1k, we report both single-crop Top-1 and Top-5 validation accuracies. 
The absolute Top-1 accuracy drop is defined as 
$\Delta \mathrm{Top1} = \mathrm{Top1}(f) - \mathrm{Top1}(\tilde f)$, 
and the Top-5 accuracy drop is defined analogously.

In all comparison tables, boldface is not used to highlight the proposed method by default. 
If boldface is used, it indicates the best value within the corresponding metric block: higher is better for pruned accuracy, FR, and PR, while lower is better for accuracy drop.

\begin{table*}[t]
    \centering
    \caption{The experimental results of VGG-16 on CIFAR-10/100.}
    \label{tab:VGG16}
    \begin{tabular}{l|l|lccccc}
        \hline
        Model & Dataset & Method & Baseline Acc. (\%) &Pruned Acc. (\%) & $\Delta$Acc$\downarrow$ (\%)& FR (\%) & PR (\%) \\
        \hline
        \multirow{22}{*}{\rotatebox{90}{VGGNet-16}}
        & \multirow{12}{*}{\centering\rotatebox{90}{CIFAR-10}}
        & SSS~\cite{SSS} & 93.96 & 93.02 & 0.94 & 41.62 & 73.77 \\
        && GAL-0.05~\cite{GAL} & 93.96 & 92.03 & 1.93 & 39.60 & 77.60 \\
        && CP~\cite{CP} & 93.59 & 93.18 & 0.41 & 39.10 & 73.30 \\
        && Zhao et al.~\cite{Zhao} & 93.83 & 93.18 & 0.65 & 39.40 & 73.30 \\
        && GAL-0.1~\cite{GAL} & 93.96 & 90.73 & 3.23 & 45.20 & 82.20 \\ 
        && CSHE~\cite{CSHE} & 93.96 & 92.00 & 1.96 & 69.00 & 82.10 \\
        && Chen et al.~\cite{Chen} & 93.91 & 93.17 & 0.74 & 52.30 & 45.70 \\
        && PCC~\cite{PCC} & 94.59 & 94.12 & 0.47 & 66.19 & 84.41 \\
        && GAS~\cite{GAS} & 93.72 & 91.66 & 2.06 & 81.15 & 93.87 \\
        \cdashline{3-8}
        && SCAP($\alpha=0.5$,$\tau=0.5$) & 93.44 & 92.18 & 1.26 & 85.88 & 93.23 \\
        && SCAP($\alpha=0.1$,$\tau=0.6$) & 93.44 & 91.46 & 1.98 & 91.27 & 96.32 \\
        \cline{2-8}
        & \multirow{10}{*}{\centering\rotatebox{90}{CIFAR-100}} 
        & Zhao et al.~\cite{Zhao} & 73.54 & 73.33 & 0.21 & 18.40 & 38.10 \\
        && FPGM~\cite{FPGM} & 73.51 & 72.26 & 1.25 & 51.01 & 51.01 \\
        && SFP~\cite{SFP} & 73.51 & 71.74 & 1.77 & 41.75 & 39.34 \\
        && HRank~\cite{HRank} & 73.51 & 72.43 & 1.08 & 41.23 & 55.93 \\
        && COP v1~\cite{COP} & 73.51 & 72.63 & 0.88 & 40.31 & 65.19 \\
        && Li et al.~\cite{Li} & 72.53 & 72.44 & 0.09 & 39.55 & 65.94 \\
        && PCC~\cite{PCC} & 74.02 & 73.94 & 0.08 & 28.57 & 27.67 \\
        && GAS~\cite{GAS} & 72.58 & 70.55 & 2.03 & 63.39 & 86.78 \\
        \cdashline{3-8}
        && SCAP($\alpha=0.5$,$\tau=0.5$) & 71.77 & 67.34 & 4.43 & 82.12 & 85.88 \\
        && SCAP($\alpha=0.3$,$\tau=0.6$) & 71.77 & 63.98 & 7.79 & 90.30 & 92.33 \\
        \hline
    \end{tabular}
\end{table*}

% =====================================================================
\subsection{Results and method comparison}
\label{sec:results}
\textbf{CIFAR results.}
Tables~\ref{tab:VGG16}--\ref{tab:DenseNet40} report the results on CIFAR-10 and CIFAR-100. Overall, SCAP consistently produces compact models across VGG-16, ResNet-56, ResNet-110, and DenseNet-40, and the two thresholds exhibit the expected accuracy--compression trade-off. On VGG-16, SCAP$(\alpha=0.5,\tau=0.5)$ achieves 85.88\%/93.23\% FR/PR on CIFAR-10 with a 1.26\% Top-1 drop, while SCAP$(\alpha=0.1,\tau=0.6)$ further increases compression to 91.27\%/96.32\% FR/PR with a 1.98\% drop. On CIFAR-100, SCAP$(\alpha=0.5,\tau=0.5)$ and SCAP$(\alpha=0.3,\tau=0.6)$ reach 82.12\%/85.88\% and 90.30\%/92.33\% FR/PR, respectively, with larger accuracy drops due to the more difficult classification setting.

For ResNet architectures, SCAP also achieves strong compression. On ResNet-56, SCAP$(\alpha=0.6,\tau=0.5)$ obtains 74.49\%/76.50\% FR/PR on CIFAR-10 and SCAP$(\alpha=0.4,\tau=0.5)$ obtains 70.49\%/74.70\% on CIFAR-100, at $\tau=0.6$, SCAP$(\alpha=0.2)$  increases the reductions to 83.70\%/84.66\% on CIFAR-10, while SCAP$(\alpha=0.3)$ reaches 82.47\%/86.83\% on CIFAR-100. On the deeper ResNet-110, SCAP$(\alpha=0.8,\tau=0.5)$ reaches 85.76\%/85.71\% FR/PR on CIFAR-10 and SCAP$(\alpha=0.3,\tau=0.6)$ reaches 89.54\%/88.93\%, with Top-1 drops of 1.74\% and 2.60\%. On CIFAR-100, SCAP$(\alpha=0.7,\tau=0.5)$ achieves 80.06\%/80.01\% FR/PR and SCAP$(\alpha=0.3,\tau=0.6)$ achieves 84.28\%/84.79\%. These results indicate that SCAP generally targets a more aggressive compression regime than many prior methods, so the comparison should be interpreted as an accuracy--compression trade-off rather than a single-metric ranking.

On DenseNet-40, SCAP remains effective despite dense inter-layer feature reuse. It achieves 75.92\%/77.87\% FR/PR on CIFAR-10 at SCAP$(\alpha=0.2,\tau=0.5)$ and 84.10\%/86.46\% at SCAP$(\alpha=0.2,\tau=0.6)$, with Top-1 drops of 1.19\% and 2.21\%, respectively. On CIFAR-100, the corresponding reductions are 74.77\%/73.66\% and 83.96\%/83.70\%, with larger drops of 4.76\% and 8.41\%. This confirms that SCAP is applicable to densely connected networks, although aggressive pruning leads to more visible degradation on the harder CIFAR-100 benchmark.

\textbf{Large-scale ImageNet-1k validation.}
To validate scalability beyond CIFAR, we further evaluate SCAP on ResNet-50 using ImageNet-1k. As shown in Table~\ref{tab:resnet50_imagenet}, SCAP$(\alpha=0.5,\tau=0.5)$ reduces FLOPs by 64.17\% and parameters by 62.12\%, while Top-1/Top-5 accuracy decreases from 76.15\%/92.87\% to 72.97\%/91.37\%. Compared with PCC and GAS, this setting provides higher compression; its Top-1 drop is smaller than PCC but larger than GAS, reflecting a stronger compression operating point. At $(\alpha=0.2,\tau=0.6)$, SCAP further increases compression to 71.07\% FR and 67.20\% PR, with Top-1/Top-5 drops of 4.58\% and 2.40\%. These ImageNet results show that the proposed spectral reconstruction criterion is not limited to low-resolution datasets and remains applicable to large-scale ResNet-50 pruning, while also exposing the expected trade-off between aggressive compression and accuracy preservation.
\begin{table*}
\centering
\caption{The experimental results of ResNet-56 on CIFAR-10/100.}
\label{tab:ResNet56}
\begin{tabular}{l|l|lccccc}
        \hline
        Model & Dataset & Method & Baseline Acc. (\%) &Pruned Acc. (\%)& $\Delta$Acc$\downarrow$ (\%) & FR (\%) & PR (\%) \\
        \hline
        \multirow{18}{*}{\rotatebox{90}{ResNet-56}}
        & \multirow{12}{*}{\centering\rotatebox{90}{CIFAR-10}}
        & CHIP~\cite{CHIP} & 93.26 & 94.16 & -0.90 & 47.40 & 42.80 \\
        && DTP~\cite{DTP} & 93.36 & 92.46 & 0.90 & 72.10 & -- \\
        && DepGraph~\cite{DepGraph} & 93.53 & 93.77 & -0.24 & 52.40 & -- \\
        && Li et al.~\cite{Li} & 93.45 & 93.07 & 0.38 & 41.96 & 18.93 \\
        && CCM~\cite{CCM} & 93.75 & 92.74 & 1.01 & 69.90 & 70.30 \\
        && CP~\cite{CP} & 93.26 & 90.80 & 2.46 & 50.60 & -- \\
        && ${FTWT}_J$~\cite{FTWT} & 93.66 & 92.28 & 1.38 & 54.00 & -- \\
        && FPGM~\cite{FPGM} & 93.59 & 93.49 & 0.10 & 52.60 & 50.60 \\
        && PCC~\cite{PCC} & 94.48 & 93.13 & 1.35 & 61.64 & 65.06 \\
       && GAS~\cite{GAS} & 93.45 & 92.11 & 1.34 & 57.52 & 65.80 \\
        \cdashline{3-8}
        && SCAP($\alpha=0.6$,$\tau=0.5$) & 93.66 & 91.30 & 2.36 & 74.49 & 76.50 \\
        && SCAP($\alpha=0.2$,$\tau=0.6$) & 93.66 & 89.96 & 3.70 & 83.70 & 84.66 \\
        \cline{2-8}
        & \multirow{6}{*}{\centering\rotatebox{90}{CIFAR-100}} 
        & Li et al.~\cite{Li}& 70.83 & 68.73 & 2.10 & 60.32 & 33.01 \\
        && PCC~\cite{PCC} & 71.45 & 65.43 & 6.02 & 72.20 & 73.97 \\
        && GAS~\cite{GAS} & 71.10 & 67.87 & 3.23 & 51.93 & 59.73 \\
          \cdashline{3-8}
        && SCAP($\alpha=0.4$,$\tau=0.5$) & 70.82 & 65.36 & 5.46 & 70.49 & 74.70 \\
        && SCAP($\alpha=0.3$,$\tau=0.6$) & 70.82 & 58.83 & 11.99 & 82.47 & 86.83 \\
        \hline
    \end{tabular}
\end{table*}

\begin{table*}
\centering
\caption{The experimental results of ResNet-110 on CIFAR-10/100}
\label{tab:ResNet110}
\setlength{\tabcolsep}{4pt}
\begin{tabular}{l|l|lccccc}
    \hline
    Model & Dataset & Method & Baseline Acc. (\%) & Pruned Acc. (\%) & $\Delta$Acc$\downarrow$ (\%) & FR (\%) & PR (\%) \\
    \hline
    \multirow{20}{*}{\rotatebox{90}{ResNet-110}}
    & \multirow{13}{*}{\centering\rotatebox{90}{CIFAR-10}} 
    & $l_1$-norm~\cite{l1} & 93.53 & 93.30 & 0.23 & 38.70 & 32.40 \\
    & & NISP~\cite{NISP} & 93.50 & 93.32 & 0.18 & 43.78 & 43.25 \\
    & & HRank~\cite{HRank} & 93.50 & 94.23 & -0.73 & 41.20 & 39.40 \\
    & & CHIP~\cite{CHIP} & 93.50 & 94.44 & -0.94 & 52.10 & 48.30 \\
    & & GAL-0.5~\cite{GAL} & 93.50 & 92.74 & 0.76 & 48.50 & 44.80 \\
    & & FPGM~\cite{FPGM} & 93.68 & 93.74 & -0.06 & 52.30 & -- \\
    & & CCM~\cite{CCM} & 94.14 & 93.80 & 0.34 & 69.40 & 68.40 \\
    & & CPRNC~\cite{CPRNC} & 93.50 & 94.16 & -0.66 & 52.60 & 52.60 \\
    & & Li et al.~\cite{Li} & 93.94 & 92.27 & 1.67 & 62.70 & 32.10 \\
    & & PCC~\cite{PCC} & 94.90 & 92.95 & 1.95 & 77.47 & 78.77 \\
    & & GAS~\cite{GAS} & 93.94 & 92.92 & 1.02 & 62.28 & 57.50 \\
    \cdashline{3-8}
    && SCAP($\alpha=0.8$,$\tau=0.5$) & 92.63 & 90.89 & 1.74 & 85.76 & 85.71 \\
    && SCAP($\alpha=0.3$,$\tau=0.6$) & 92.63 & 90.03 & 2.60 & 89.54 & 88.93 \\
    \cline{2-8}
    & \multirow{7}{*}{\centering\rotatebox{90}{CIFAR-100}}
    & $l_1$-norm~\cite{l1} & 70.95 & 69.88 & 1.07 & 52.42 & 20.72 \\
    & & Li et al.~\cite{Li} & 70.95 & 70.27 & 0.68 & 52.42 & 20.72 \\
    & & PCC~\cite{PCC} & 73.69 & 68.02 & 5.67 & 74.63 & 75.72 \\
    & & GAS~\cite{GAS} & 72.43 & 69.88 & 2.55 & 54.00 & 56.43 \\
    \cdashline{3-8}
    && SCAP($\alpha=0.7$,$\tau=0.5$) & 71.26 & 65.71 & 5.55 & 80.06 & 80.01 \\
    && SCAP($\alpha=0.3$,$\tau=0.6$) & 71.26 & 63.42 & 7.84 & 84.28 & 84.79 \\
    \hline
    \end{tabular}
\end{table*}

\begin{table*}
\centering
\caption{The experimental results of DenseNet-40 on CIFAR-10/100.}
\label{tab:DenseNet40}
\begin{tabular}{l|l|lccccc}
    \hline
        Model & Dataset & Method & Baseline Acc. (\%) & Pruned Acc. (\%) & $\Delta$Acc$\downarrow$ (\%) & FR (\%) & PR (\%) \\
    \hline
    \multirow{15}{*}{\rotatebox{90}{DenseNet-40}} 
    & \multirow{9}{*}{\rotatebox{90}{CIFAR-10}} 
    & GAL-0.01~\cite{GAL} & 94.81 & 94.29 & 0.52 & 35.13 & 35.58 \\
    & & Li et al.~\cite{Li} & 94.19 & 93.00 & 1.19 & 43.09 & 24.03 \\
    & & Zhao et al.~\cite{Zhao} & 94.11 & 93.16 & 0.95 & 44.78 & 59.67 \\
    & & HRank~\cite{HRank} & 94.81 & 93.68 & 1.13 & 60.94 & 53.85 \\
    & & DCT~\cite{DCT} & 94.81 & 94.32 & 0.49 & 38.50 & 40.40 \\
    & & Chen et al.~\cite{Chen} & 94.22 & 93.56 & 0.66 & 44.40 & 60.08 \\
    & & GAS~\cite{GAS} & 93.88 & 93.11 & 0.77 & 52.10 & 61.18 \\
    \cdashline{3-8}
    & & SCAP($\alpha=0.2$,$\tau=0.5$) & 93.72 & 92.53 & 1.19 & 75.92 & 77.87 \\
    & & SCAP($\alpha=0.2$,$\tau=0.6$) & 93.72 & 91.51 & 2.21 & 84.10 & 86.46 \\
    \cline{2-8}
    & \multirow{6}{*}{\rotatebox{90}{CIFAR-100}} 
    & Zhao et al.~\cite{Zhao} & 74.64 & 72.19 & 2.45 & 22.67 & 37.73 \\
    & & Li et al.~\cite{Li} & 72.85 & 72.07 & 0.78 & 41.87 & 18.22 \\
    & & $l_1$-norm~\cite{l1}  & 72.85 & 71.89 & 0.96 & 41.87 & 18.22 \\
    & & GAS~\cite{GAS} & 72.71 & 70.04 & 2.67 & 58.97 & 60.25 \\
    \cdashline{3-8}
    & & SCAP($\alpha=0.2$,$\tau=0.5$) & 72.95 & 68.19 & 4.76 & 74.77 & 73.66 \\
    & & SCAP($\alpha=0.2$,$\tau=0.6$) & 72.95 & 64.54 & 8.41 & 83.96 & 83.70 \\
    \hline
\end{tabular}
\end{table*}

\begin{table*}
\centering
\caption{The experimental results of ResNet-50 on ImageNet-1k. 
Baseline and pruned accuracy are reported as Top-1/Top-5. 
FR and PR denote FLOP reduction and parameter reduction, respectively.}
\label{tab:resnet50_imagenet}
\setlength{\tabcolsep}{4pt}
\renewcommand{\arraystretch}{1.08}
\begin{tabular}{c|lcccccc}
\hline
Model & Method & Baseline Acc. (\%) & Pruned Acc. (\%) & $\Delta$Top-1 Acc$\downarrow$ (\%) & $\Delta$Top-5 Acc$\downarrow$ (\%)& FR (\%) & PR (\%) \\
\hline
\multirow{15}{*}{\rotatebox[origin=c]{90}{ResNet-50}}
    & ThiNet~\cite{ThiNet} & 72.88/91.14 & 72.04/90.67 & 0.84 & 0.47 & 36.79 & 33.72 \\
    & SSS~\cite{SSS} & 76.00/92.90 & 74.03/91.94 & 1.97 & 0.96 & 31.05 & 27.06 \\
    & GAL-0.5~\cite{GAL} & 76.15/92.87 & 71.95/90.94 & 4.20 & 1.93 & 43.03 & 16.86 \\
    & FPGM~\cite{FPGM} & 76.15/92.87 & 75.59/92.63 & 0.56 & 0.24 & 42.20 & 37.50 \\
    & HRank~\cite{HRank} & 76.15/92.87 & 74.76/92.33 & 1.39 & 0.54 & 43.70 & 36.60 \\
    & SRR-GR~\cite{SRR-GR} & 76.13/92.86 & 75.76/92.67 & 0.37 & 0.19 & 44.10 & -- \\
    & DepGraph~\cite{DepGraph} & 76.15/-- & 75.83/-- & 0.32 & -- & 51.80 & -- \\
    & DTP~\cite{DTP} & 76.13/-- & 75.55/-- & 0.58 & -- & 56.70 & -- \\
    & Li et al.~\cite{Li} & 76.15/92.87 & 74.68/92.04 & 1.47 & 0.83 & 48.81 & 23.71 \\
    & PCC~\cite{PCC} & 76.87/93.14 & 73.33/91.49 & 3.54 & 1.65 & 58.40 & 60.24 \\
    & GAS~\cite{GAS} & 76.15/92.87 & 74.04/91.80 & 2.11& 1.07 & 60.58 & 60.97 \\
    \cdashline{3-8}
    & SCAP($\alpha=0.5$,$\tau=0.5$) & 76.15/92.87 & 72.97/91.37 & 3.18 & 1.50 & 64.17 & 62.12 \\
    & SCAP($\alpha=0.2$,$\tau=0.6$) & 76.15/92.87 & 71.57/90.47 & 4.58 & 2.40 & 71.07 & 67.20 \\
\hline
\end{tabular}
\end{table*}

% =====================================================================

\subsection{Budget-Matched Comparison under Aggressive Compression}
\label{sec:aggressive_compression}
To further assess the accuracy of SCAP degradation under aggressive compression, we conduct a budget-matched comparison on CIFAR-100 by re-implementing representative pruning criteria (magnitude-based, geometric, Taylor/gradient-based, and GAS-based) under a unified protocol, matching the FLOP of SCAP  and parameter reductions at $\tau$=0.6. This ensures fair comparison of accuracy preservation at the same structural compression level. Results for other backbones are provided in the supplementary material.
As shown in Table~\ref{tab:budget_matched_aggressive}, SCAP remains competitive when representative criteria are evaluated under the same FR/PR constraints. This suggests that the larger accuracy degradation observed at $\tau=0.6$ is mainly associated with the extreme compression budget rather than an unfair comparison protocol.

\begin{table}[t]
\centering
\caption{Budget-matched comparison under aggressive compression on CIFAR-100. Lower $\Delta$Acc indicates better accuracy preservation under the same compression level.}
\label{tab:budget_matched_aggressive}
\setlength{\tabcolsep}{3.0pt}
\renewcommand{\arraystretch}{1.0}
\begin{tabular}{llccc}
\toprule
Network & Method & Pruned Acc. (\%) & $\Delta$Acc$\downarrow$ (\%) & FR / PR (\%) \\
\midrule
\multirow{5}{*}{\shortstack{VGG-16\\71.77\%}}
 & $\ell_1$  & 63.45 & 8.32 & \multirow{5}{*}{90.30 / 92.33} \\
 & FPGM      & 63.28 & 8.49 & \\
 & Taylor    & 61.79 & 9.98 & \\
 & GAS       & 61.79 & 9.98 & \\
 & SCAP      & \textbf{63.98} & \textbf{7.79} & \\
\midrule
\multirow{5}{*}{\shortstack{DenseNet-40\\72.95\%}}
 & $\ell_1$  & 64.28 & 8.67 & \multirow{5}{*}{83.96 / 83.70} \\
 & FPGM      & 64.48 & 8.47 & \\
 & Taylor    & 64.36 & 8.59 & \\
 & GAS       & 64.31 & 8.64 & \\
 & SCAP      & \textbf{64.54} & \textbf{8.41} & \\
\bottomrule
\end{tabular}
\end{table}

% =====================================================================
\subsection{Ablation Study}
\label{sec:ablation}
The default SCAP score combines the fidelity-derived term $I_{\mathrm{fid}}$ and the normalized set-$\ell_1$ magnitude term $I_{\mathrm{l1}}$. The former measures interaction-level spectral compressibility, while the latter provides magnitude-aware saliency to compensate for the scale-invariance of cosine fidelity. Since $I_{\mathrm{fid}}$ does not encode absolute response magnitude, a channel with a highly reconstructable spectral pattern may still be important if its magnitude is large. Therefore, we compare full SCAP with a fidelity-only variant to examine the necessity of the magnitude-aware component.

Specifically, \textbf{Fidelity-only} uses $I_{\mathrm{fid}}$ alone, i.e., $\alpha=1$ in Eq.~\eqref{eq:importance_fusion}, whereas \textbf{SCAP} uses the default additive fusion of $I_{\mathrm{fid}}$ and $I_{\mathrm{l1}}$. Table~\ref{tab:ablation} reports the results on CIFAR-100 under thresholds $\tau\in\{0.5,0.6\}$ across four backbones. This is an equal-threshold ablation rather than an equal-FLOP comparison, since different importance distributions may lead to different compression ratios under the same threshold. Additional ablation results are provided in the supplementary material.

The fidelity-only variant produces very high compression, confirming that spectral reconstructability is a meaningful pruning signal. However, it is consistently over-aggressive. For example, on ResNet-110 at $\tau=0.6$, it removes $97.40\%$ of parameters and $94.99\%$ of FLOPs, but the pruned model retains only $33.61\%$ Top-1 accuracy, corresponding to a $37.65$ percentage-point drop. In contrast, adding $I_{\mathrm{l1}}$ substantially improves the accuracy--compression trade-off. Across all eight settings, SCAP reduces the accuracy drop by approximately $2.3\times$ to $5.3\times$ compared with fidelity-only pruning while maintaining strong compression. For instance, on ResNet-56 at $\tau=0.5$, the drop decreases from $26.35$ to $5.46$ percentage points, with $70.49\%$ FLOP reduction and $74.70\%$ parameter reduction.

These results show that fidelity captures interaction-level spectral compressibility, but alone it is insufficient as a complete saliency measure under aggressive pruning. The set-$\ell_1$ term acts as a complementary magnitude-aware stabilizer rather than an unrelated empirical patch. The sensitivity of the fusion weight $\alpha$ is further analyzed in Section~\ref{sec:alpha_sensitivity}.

\begin{table*}[t]
\centering
\caption{Ablation study on CIFAR-100: fidelity-only scoring versus full SCAP scoring.
FR and PR denote FLOP reduction and parameter reduction, respectively.
The comparison is conducted under the same pruning threshold $\tau$ and
fine-tuning protocol.}
\label{tab:ablation}
\setlength{\tabcolsep}{5pt}
\renewcommand{\arraystretch}{1.08}
\begin{tabular}{llcccccc}
\toprule
Network & Method & $\tau$ & Baseline Acc. (\%) & Pruned Acc. (\%) & $\Delta$Acc $\downarrow$ (\%) & FR (\%) & PR (\%) \\
\midrule
\multirow{4}{*}{ResNet-56}
  & Fidelity-only & 0.5 & 70.82 & 44.47 & 26.35 & 86.07 & 93.47 \\
  & SCAP          & 0.5 & 70.82 & 65.36 &  5.46 & 70.49 & 74.70 \\
  & Fidelity-only & 0.6 & 70.82 & 37.94 & 32.88 & 91.33 & 96.07 \\
  & SCAP          & 0.6 & 70.82 & 58.83 & 11.99 & 82.47 & 86.83 \\
\midrule
\multirow{4}{*}{ResNet-110}
  & Fidelity-only & 0.5 & 71.26 & 41.67 & 29.59 & 93.47 & 96.04 \\
  & SCAP          & 0.5 & 71.26 & 65.71 &  5.55 & 80.06 & 80.01 \\
  & Fidelity-only & 0.6 & 71.26 & 33.61 & 37.65 & 94.99 & 97.40 \\
  & SCAP          & 0.6 & 71.26 & 63.42 &  7.84 & 84.28 & 84.79 \\
\midrule
\multirow{4}{*}{DenseNet-40}
  & Fidelity-only & 0.5 & 72.95 & 59.40 & 13.55 & 85.79 & 88.77 \\
  & SCAP          & 0.5 & 72.95 & 68.19 &  4.76 & 74.77 & 73.66 \\
  & Fidelity-only & 0.6 & 72.95 & 53.17 & 19.78 & 88.57 & 91.86 \\
  & SCAP          & 0.6 & 72.95 & 64.54 &  8.41 & 83.96 & 83.70 \\
\midrule
\multirow{4}{*}{VGG-16}
  & Fidelity-only & 0.5 & 71.77 & 56.20 & 15.57 & 89.24 & 87.06 \\
  & SCAP          & 0.5 & 71.77 & 67.34 &  4.43 & 82.12 & 85.88 \\
  & Fidelity-only & 0.6 & 71.77 & 49.31 & 22.46 & 94.05 & 92.23 \\
  & SCAP          & 0.6 & 71.77 & 63.98 &  7.79 & 90.30 & 92.33 \\
\bottomrule
\end{tabular}
\end{table*}

\subsection{Sensitivity to the fusion weight}
\label{sec:alpha_sensitivity}
We further analyze the sensitivity of SCAP to the additive fusion weight $\alpha$ in Eq.~\eqref{eq:importance_fusion}. Here $\alpha=0$ corresponds to normalized $\ell_1$-magnitude pruning, $\alpha=1$ corresponds to fidelity-only pruning, and intermediate values combine magnitude-aware saliency with spectral reconstructability. In the main paper, we report the sensitivity curves under the representative threshold $\tau=0.5$; the complete numerical tables and the corresponding $\tau=0.6$ results are provided in the supplementary material.

\begin{figure}
\centering
\subfloat[VGG-16]{\includegraphics[width=1.65in]{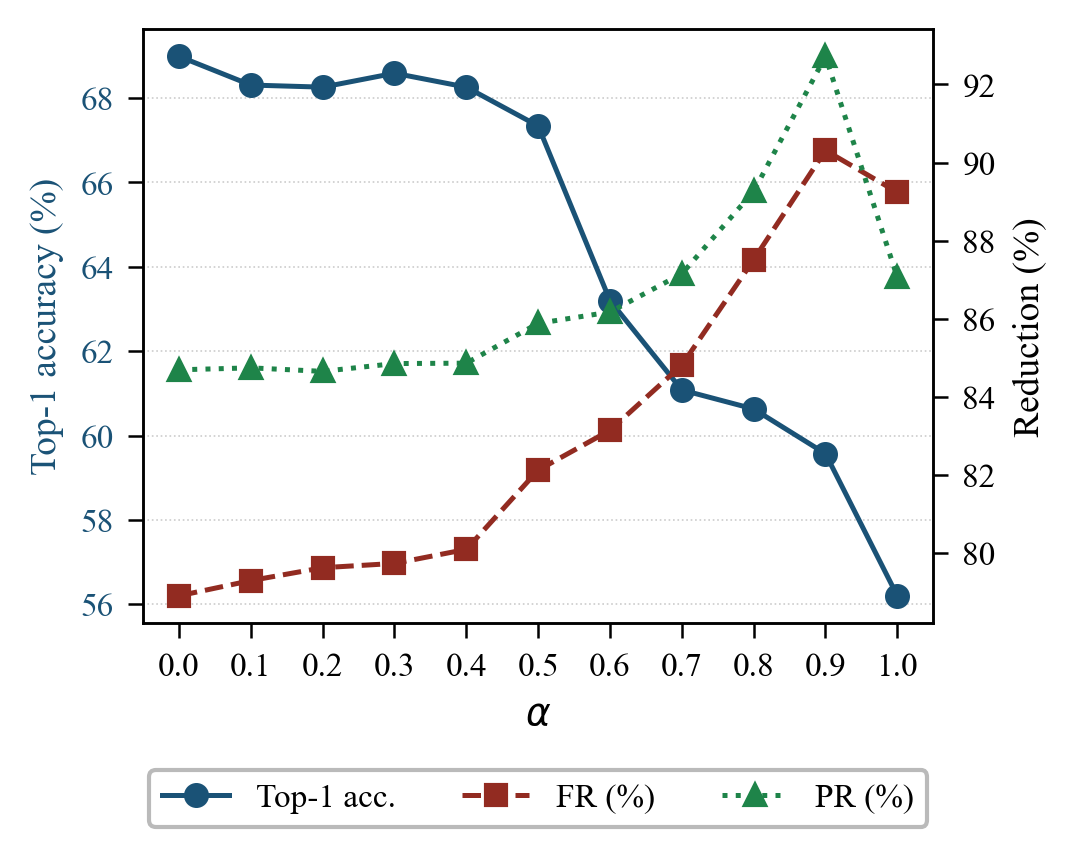}}
\hfil
\subfloat[ResNet-56]{\includegraphics[width=1.65in]{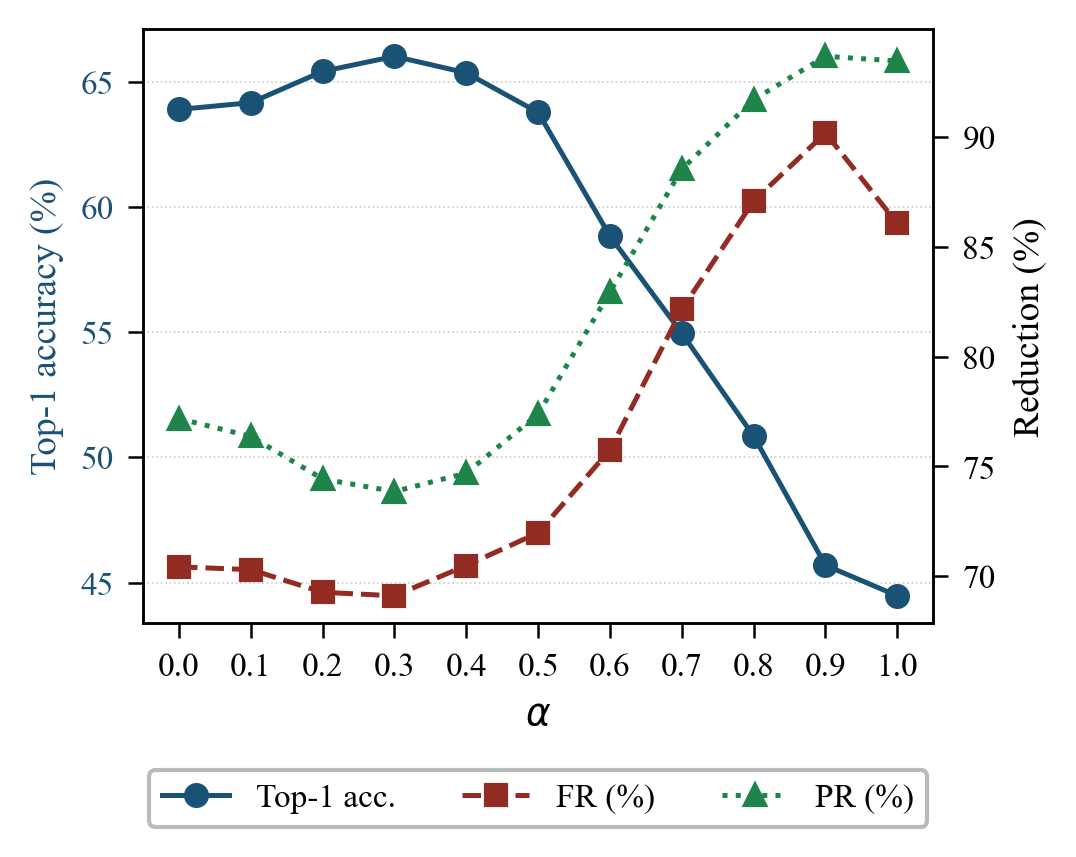}}
\\[-2pt]
\subfloat[ResNet-110]{\includegraphics[width=1.65in]{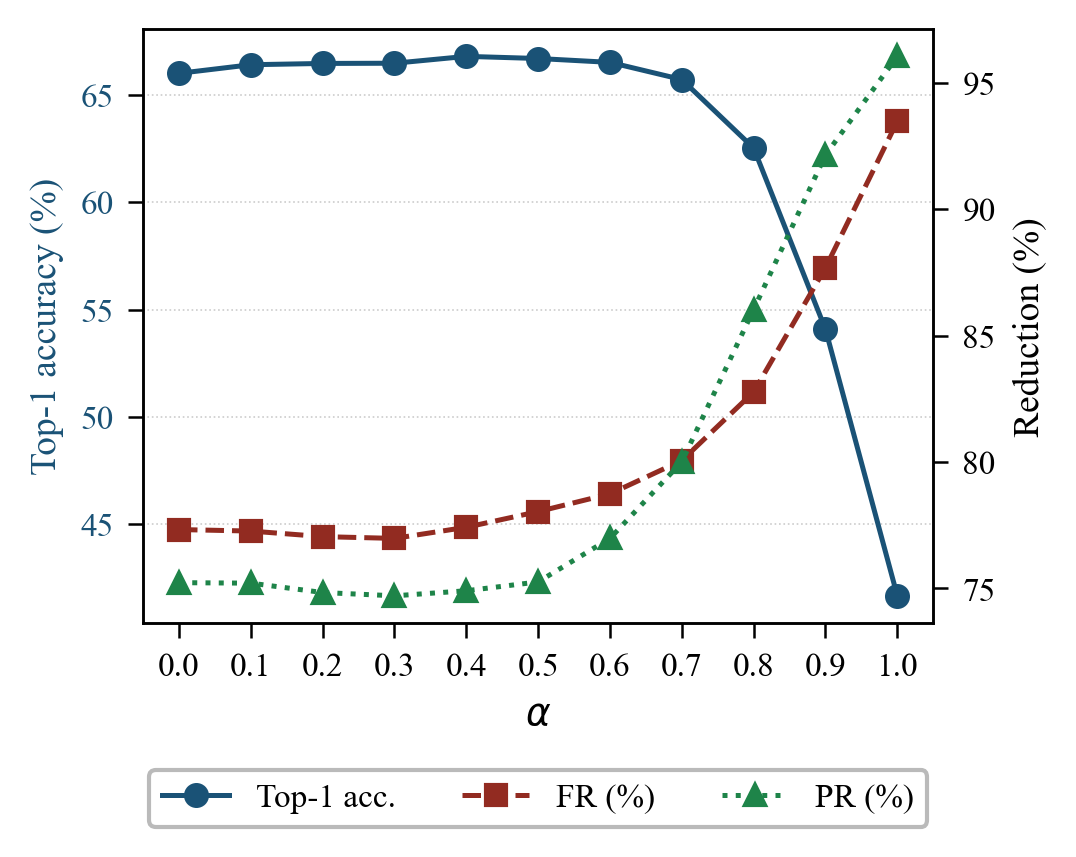}}
\hfil
\subfloat[DenseNet-40]{\includegraphics[width=1.65in]{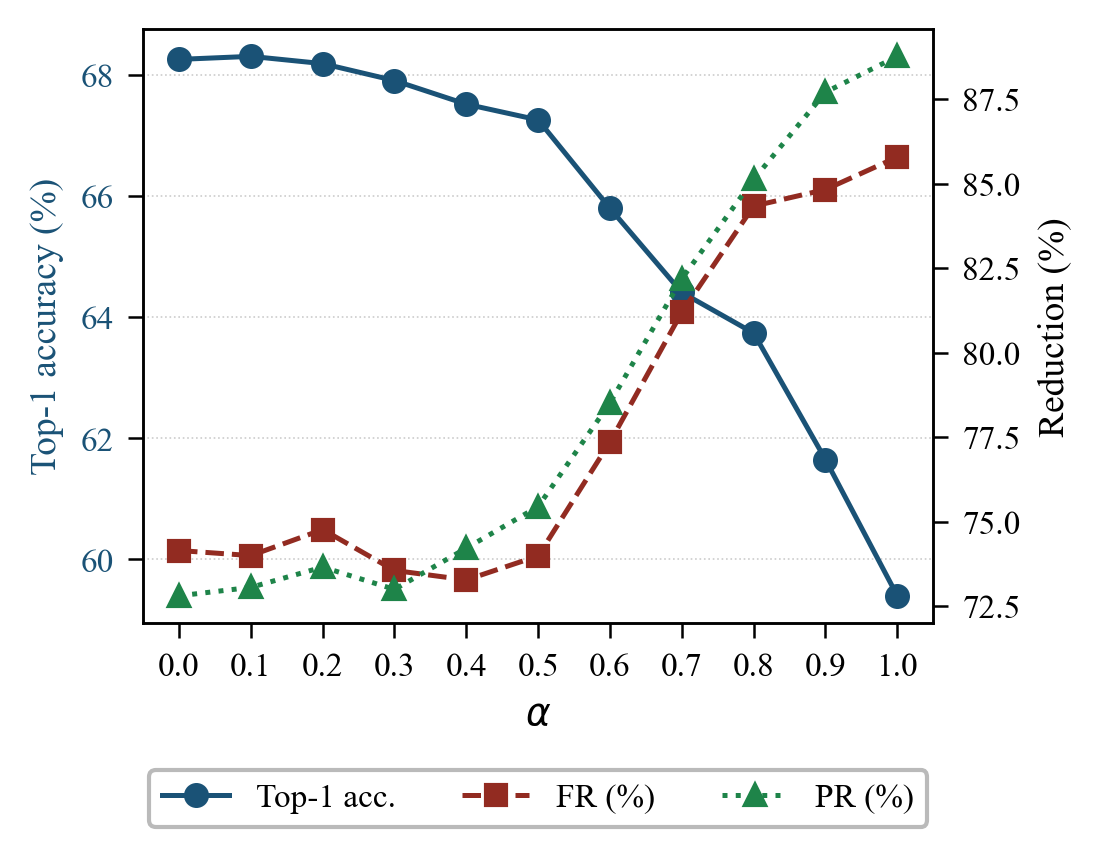}}
\caption{Sensitivity of SCAP to the fusion weight $\alpha$ under threshold $\tau=0.5$ on CIFAR-100. Each subfigure reports Top-1 accuracy, FLOP reduction (FR), and parameter reduction (PR) as $\alpha$ varies from $0$ to $1$. }
\label{fig:alpha_sensitivity_tau05}
\end{figure}

Figure~\ref{fig:alpha_sensitivity_tau05} shows that SCAP is sensitive to the balance between the fidelity-derived mismatch and the set-$\ell_1$ magnitude term. Small-to-moderate $\alpha$ values generally provide stable accuracy--compression trade-offs, whereas overly large $\alpha$ values make the criterion dominated by normalized spectral reconstructability and may lead to excessive compression with sharp accuracy degradation. This trend confirms that fidelity-derived mismatch is useful but should not replace magnitude-aware saliency. Instead, the additive fusion balances interaction-level spectral information with filter-magnitude information, which is why SCAP uses the fused score as its default criterion.

\subsection{Computational Overhead}
\label{sec:overhead}
SCAP incurs an offline preparation cost: training lightweight spectral autoencoders per layer. We compare its score-computation time with representative criteria under a unified protocol on CIFAR-100 (same backbones VGG-16, ResNet-56/110, DenseNet-40, calibration size 512, excluding fine-tuning). Table~\ref{tab:overhead} lists average times. Weight-only criteria (e.g., $\ell_1$, FPGM) need no data; activation-statistics require forward passes; HRank additionally estimates ranks; Gradient/Taylor need forward and backward passes. SCAP requires activation collection, FFT descriptor construction, per-layer AE training, and scoring, averaging 6618.36~s across the four backbones—substantially higher. This overhead is offline and one-time; the AEs are discarded after pruning, leaving no extra inference cost. In main experiments, each AE trains up to 100 epochs with early stopping (patience 5); the supplementary material presents a 20-epoch lightweight variant.

\begin{table}
\centering
\caption{Preparation overhead comparison on CIFAR-100 under a unified calibration protocol. We report the average preparation time across four backbones: VGG-16, ResNet-56, ResNet-110, and DenseNet-40. Fine-tuning time is excluded. ``Model fwd.'' indicates whether calibration data must be passed through the pretrained network, and ``Criterion bwd.'' indicates whether the criterion itself requires back-propagation.}
\label{tab:overhead}
\renewcommand{\arraystretch}{1.10}
\setlength{\tabcolsep}{2.0pt}
\begin{tabular}{lcccc}
\toprule
Criterion & Model fwd. & Criterion bwd. & Aux. training  & Prep. time \\
\midrule
$\ell_1$ magnitude& No & No & No& 0.64 s \\
FPGM & No & No & No  & 4.60 s \\
Activation statistics & Yes & No & No  & 8.49 s \\
HRank & Yes & No & No  & 9.69 s \\
Gradient/Taylor & Yes & Yes & No & 8.27 s \\
SCAP & Yes & Yes & Per-layer AE & 6618.36 s \\
\bottomrule
\end{tabular}
\end{table}

% =====================================================================
\section{Conclusion and Limitations}
\label{sec:conclusion}

We presented \emph{Spectral Complex Autoencoder Pruning} (SCAP), a structured channel pruning framework that estimates channel redundancy from the spectral reconstructability of input--output interaction descriptors. 
For each convolutional layer, SCAP constructs a complex interaction field from the layer input and each output-channel response, transforms it into the frequency domain, and uses a low-capacity autoencoder to obtain a fidelity-derived reconstruction mismatch. 
This mismatch is used as a descriptor-level compressibility cue, not as a standalone saliency measure.

SCAP combines this scale-invariant reconstructability cue with a normalized set-$\ell_1$ filter-magnitude term, yielding a fused importance score that balances interaction-level structure and magnitude-aware saliency. 
Empirical results on CIFAR and ImageNet demonstrate strong FLOP and parameter reductions under aggressive compression, and the DeiT-Tiny experiment provides preliminary evidence that the same interaction-reconstructability principle can be instantiated beyond conventional CNN filters. 
Budget-matched comparisons further show that the high-compression accuracy degradation is largely tied to extreme FR/PR budgets, with SCAP remaining competitive under matched compression constraints.

\paragraph{Limitations.}
SCAP is a layer-local pruning criterion. 
It scores structural units within each layer according to the reconstructability of their input--output interaction descriptors, but it does not solve a global combinatorial optimization problem over all layers. 
Since removing channels changes subsequent activations, normalization behavior, and cross-layer feature dependencies, the layer-wise scores do not guarantee globally optimal channel selection or global accuracy preservation. 
Standard fine-tuning is therefore required to adapt the remaining parameters to the modified architecture. 
In addition, our theoretical analysis links fidelity to normalized descriptor reconstruction mismatch and bounded perturbation of aligned channel maps, but it does not provide an end-to-end classification accuracy bound. 
SCAP also introduces offline preparation overhead due to layer-wise spectral autoencoder training, although these autoencoders are discarded after pruning and add no inference-time cost. Recent task-level efficiency studies have also explored pruning or lightweight modeling for medical image segmentation, small-cell detection in whole-slide images, and spatiotemporal forecasting~\cite{han2025medical_pruning,yu2026small_cell,wang2025tcfi,wang2026s2kd}. For example, feature-group clustering with spectral variance rebalancing and topology-consistent Fisher pruning have been studied for medical image segmentation, while semantic-spectral knowledge distillation has been explored for efficient spatiotemporal forecasting. These directions suggest that extending SCAP from classification-oriented backbone pruning to task-level structured pruning for segmentation, detection, and forecasting is a meaningful future direction. Extending SCAP to globally coordinated pruning, larger model families, non-classification downstream tasks, and hardware-specific latency, FPS, power, and edge-deployment measurements remains future work.

\section*{Acknowledgment}
This work was supported by Postgraduate Research \& Practice Innovation Program of Jiangsu Province (Grant No. 26CXJH7393), under the project titled "Research on Spectral Complex Autoencoder Structured Pruning Methods" and by the National Natural Science Foundation of China (Grant No. 62202210), under the project titled "Research on the Generative Mechanism of Patent Causal Features for the Robust Quantification of Technology and Appearance".

\bibliographystyle{IEEEtran}   % 可换 IEEEtran、abbrv 等
\bibliography{refs}

\end{document}